\documentclass[nohyperref]{article}

\usepackage[english]{babel}

\usepackage{microtype}
\usepackage{graphicx}
\usepackage{booktabs} %

\usepackage[accepted]{icml2023}

\usepackage{amsmath}
\usepackage{amssymb}
\usepackage{mathtools}
\usepackage{amsthm}

\usepackage{url}

\usepackage{amsmath}
\usepackage{graphicx}
\usepackage[colorlinks=true, allcolors=blue]{hyperref}
\usepackage{amsthm}
\usepackage{amssymb}
\usepackage{amsfonts}
\usepackage{mathrsfs}
\usepackage{multirow,array}
\usepackage{xcolor}
\usepackage{amsmath}
\usepackage{tikz, pgfplots}
\pgfplotsset{compat=1.18}
\usepackage{float}

\usepackage{subcaption}
\usepackage{enumitem}  

\usepackage{comment} 

\usepackage{cleveref}

\def\states{\mathcal{S}}
\def\actions{\mathcal{A}}
\def\obss{\Omega} %
\def\obsf{\mathcal{O}} %

\def\br{\text{BR}}

\def\SECstructure{\mathcal{G}}

\def\samp{\sim}
\def\E{\mathop{\mathbb{E}}}

\def\R{\mathbb{R}}
\newcommand{\powerset}[1]{\mathcal{P}(#1)}

\theoremstyle{definition}
\newtheorem{definition}{Definition}[section]

\newtheorem{remark}{Remark}
\newtheorem{theorem}{Theorem}
\newtheorem{lemma}{Lemma}

\DeclareMathOperator*{\argmax}{arg\,max}

\icmltitlerunning{Who Needs to Know? Minimal Knowledge for Optimal Coordination}

\begin{document}

\twocolumn[
\icmltitle{Who Needs to Know? Minimal Knowledge for Optimal Coordination}

\begin{icmlauthorlist}
\icmlauthor{Niklas Lauffer}{ucb}
\icmlauthor{Ameesh Shah}{ucb}
\icmlauthor{Micah Carroll}{ucb}
\icmlauthor{Michael Dennis}{ucb}
\icmlauthor{Stuart Russell}{ucb}
\end{icmlauthorlist}

\icmlaffiliation{ucb}{Department of Electrical Engineering and Computer Science, University of California, Berkeley, CA, USA}

\icmlcorrespondingauthor{Niklas Lauffer}{nlauffer@berkeley.edu}

\icmlkeywords{Machine Learning, ICML, artificial intelligence, game theory, multiagent learning, coordination, collaboration}

\vskip 0.3in
]

\printAffiliationsAndNotice{}  %

\begin{abstract}
To optimally coordinate with others in cooperative games, it is often crucial to have information about one’s collaborators: successful driving requires understanding which side of the road to drive on. However, not every feature of collaborators is \emph{strategically relevant}: the fine-grained acceleration of drivers may be ignored while maintaining optimal coordination. We show that there is a well-defined dichotomy between strategically relevant and irrelevant information. Moreover, we show that, in dynamic games, this dichotomy has a compact representation that can be efficiently computed via a Bellman backup operator. We apply this algorithm to analyze the strategically relevant information for tasks in both a standard and a partially observable version of the Overcooked environment. Theoretical and empirical results show that our algorithms are significantly more efficient than baselines. Videos are available at \url{https://minknowledge.github.io}.
\end{abstract}

\section{Introduction}

When designing a policy for a cooperative multi-agent setting, it is often critical to have some idea of how one's co-players will behave. A policy for a self driving car must take into account various driving norms such as which side of the street other cars will drive on, and how to interpret stop light signals.  

While there are many relevant features of co-player behavior that a policy designer must keep in mind, there are often many more irrelevant features that can be safely ignored. It is unnecessary for a self driving car to know the final destination of every other car, the current positions or trajectories of far away cars, or their idiosyncratic driving behaviors.  

In such settings, it can be useful to separate the \emph{strategically relevant} information from the \emph{strategically irrelevant} information.
If much of the information is irrelevant, it’s easy to imagine how this could lead to significant increases in efficiency for finding optimal policies. For example, this could allow a focused effort on few-shot or zero-shot adaptation to co-players \cite{Zand2022OntheflySA, albrecht2017reasoning, stone2010adhoc, hu2021otherplay} or more efficient DecPOMDP planning algorithms \cite{szer2006pointbased, seuken2007memory}. In order to leverage these benefits, we build the theory, data structures, and algorithms required to distinguish between relevant and irrelevant information.

We formalize the idea of the strategically relevant information via the idea of \emph{strategic ambiguity}, described in Section \ref{sec:strategic_relevance}. We show that this leads to a uniquely well-defined dichotomy between the strategically relevant and irrelevant information, captured by the 
\emph{strategic equivalence relation} (SER) which defines two co-player policies as being \textit{strategically equivalent} if and only if they have the same set of best-response policies. 

We show that this strategic equivalence relation can be efficiently computed and stored in Section \ref{sec:compute_convs}. In fact, we find that our proposed algorithm for computing these strategic equivalence relations has better computational complexity than the state of the art for finding best-response policies in DecPOMDPs.

We summarize the contributions of this paper as follows:
\vspace{-1em}
\begin{enumerate}\setlength\itemsep{-.2em}
    \item We formalize the dichotomy between strategically relevant and irrelevant information by introducing the concept of a strategic equivalence relation (SER). 
    \item We demonstrate and prove that SERs have a recursive substructure in dynamic games and show how they can be stored compactly in the form of a directed acyclic graph. 
    \item We provide novel algorithms for computing SERs in both fully-observed and partially-observed common-payoff stochastic games. These algorithms also have notable applications in more efficiently solving DecPOMDPs. 
    \item We demonstrate what SERs look like in various problems and use them to analyze the levels of coordination required in \textit{Overcooked}, a popular benchmark for testing coordination in the AI literature.
\end{enumerate}

\section{Related Work} \label{sec:related}

There have been several lines of research towards getting AI agents to optimally coordinate in fully-cooperative multi-agent problems. Many works have explored designing multi-agent reinforcement learning methods to train agents in a centralized way to converge on Nash equilibrium in common-payoff games \cite{zhang2019multiagentrl, NIPS2016_c7635bfd, jin2022v, wang2002reinforcement}. Others have explored the setting of \textit{ad-hoc teamwork}, in which agents aim to coordinate with teammates from a target population~\cite{barrett15adhoc, suriadinata21adhoc}, and \textit{zero-shot coordination}, where agents need aim to coordinate with arbitrary teammates without any prior coordination~\cite{hu2021otherplay, treutlein2021new, muglich2022equivariant}. Our work aims to accelerate these lines of research by providing a framework that focuses on the strategically relevant aspects of the coordination problem. %

A distinct line of research has been in the direction of analyzing and understanding \textit{conventions} that agents use to overcome coordination problems \cite{lewis1975conv, dylan2018silly}. We show that strategic equivalence relations naturally divide policies into the different conventions that agents could follow in a problem.
Various measures \cite{FontaineB-RSS-21, fontaine2021diverse} have been developed to analyze the levels of coordination required in multiagent problems. However, all of these approaches are based on simulating agents' behaviors rather than directly analyzing the problem.

Most close to our line of research are some works focused on understanding what the relevant pieces of information are in different decision making problems. \citet{pynadath2007minimal} discuss the idea of minimal models of belief that agents need to solve decision making problems. \citet{becker2003transition, witwicki2010influence, oliehoe2021influence} use the insight that it is unnecessary for every agent to know everything about other agents and focus on how agents locally influence each other by exploiting the structure of particular classes of games. Several works explore settings in which games have preexisting \textit{influence diagrams} that describe which decision variables are dependent on each other \cite{koller2003diagrams, kearns2013graphicalgames, mura2000game}. In contrast, our framework for computing strategically relevant information requires no assumptions or structure in how agents interact or influence each other within the decision making problem.

We find ourselves in the setting of cooperative (common-payoff) partially observable stochastic games, also known as decentralized partially observable Markov decision processes (DecPOMDPs). Our algorithms pull ideas from existing work on solving DecPOMDPs \cite{nair2003jesp, szer2006pointbased, seuken2007memory, dibangoye2016optimally}.
We also draw inspiration from \citet{halpern2004knowledge} and model knowledge as \textit{possible worlds}. That is, we represent the knowledge that agents have about co-policies through the subset of possible co-policies that they cannot tell apart.

\section{Preliminaries} \label{sec:prelim}

We consider the setting in which agents play in a common-payoff \textit{stochastic game}, also known as a \textit{decentralized partially observable Markov decision process} (DecPOMDP).

\begin{definition}
    A \textit{decentralized partially observable Markov decision process} is a tuple $(M, \states, \actions, R, \gamma, P, \obss, \obsf)$ whose elements are defined as follows.
    \vspace{-1em}
    \begin{itemize}\setlength\itemsep{-.2em}
    \item $M = \{1,\dots,m\}$ is a set of agents.
    \item $\states$ is a set of states.
    \item $\actions = \actions_1 \times \dots \times \actions_m$ is the space of joint actions. For ease of notation, we assume without loss of generality that the action set is identical across states.
    \item $R : \states \times \actions \times \states \to \R$  is the common reward function.
    \item $\gamma$ is the discount factor.
    \item $P : \states \times \actions \times \states \to [0,1]$ is the transition function that satisfies $\sum_{s'\in \states}P(s,a,s')=1.$ 
    \item $\obss = \obss_1 \times \dots \times \obss_m$ is the joint observation space.
    \item $\obsf : \states \times \actions \times \obss \to [0,1]$ is the observation function.
    \end{itemize}
\end{definition}

A DecPOMDP is \textit{fully-observed} if $\obss_i = \states$, i.e., agents always observe the current state. We only consider \textit{finite-horizon} DecPOMDPs with episode length $T$.

A stationary policy for agent $i$ in a DecPOMDP specifies a distribution over actions $\actions_i$ at every state for every possible history $h^t_i = [(a^1_i, o^1_i), (a^2_i, o^2_i), \dots, (a^t_i, o^t_i)]$ of play. We often interpret policies as \textit{trees}, where the nodes are distributions over actions and the edges are observations. We say a policy is \textit{pure} if its choice of distribution over actions is deterministic everywhere. A \textit{Markov policy} is a policy where the action distribution depends only on the latest observation (or state). Let $\Pi$ denote the joint policy space for all agents, $\Pi_{i}$ denote the policy space of agent $i$, and $\Pi_{-i}$ denote the \textit{co-policy} space, the joint policy space of the co-players (all agents other than $i$).

\paragraph{Game theory.}
As an abuse of notation, we will sometimes write $R(\pi)$ or $R(\pi_i, \pi_{-i})$ to denote the expected return of joint policy $\pi = (\pi_i, \pi_{-i})$.
\begin{definition}
  The \textit{best-response function} for agent $i$ is the set-valued function $\br_i: \Pi_{-i} \to \powerset{\Pi_{i}}$ such that $\br_{i}(\pi_{-i}) = \argmax_{\pi_i} R(\pi_i, \pi_{-i})$, denoting the set of policies agent $i$ could play to maximize payoff in response to the co-policy $\pi_{-i}$. 
  The individual best-response functions for each agent can be combined to form the joint best-response function $\br$ for all agents,
 where 
$\br (\pi) = (\br_{1}(\pi_{-1}), \br_{2}(\pi_{-2}), \dots, \br_m(\pi_{-m}))$.
\end{definition} 
We call the policies $\bigcup_{\pi_{-i} \in \Pi_{-i}} \br(\pi_{-i})$, the set of \emph{best-response policies} for player $i$.

In our context, a joint policy $\pi$ is a \emph{subgame perfect equilibrium} if, for all states $s \in S$,
\begin{equation} \label{eq:nash_eq}
    Q^{\pi}(s,  \pi(s)) \geq Q^{\pi}(s, a_i, \pi_{-i}(s)), \quad \forall a_i \in \actions_i,
\end{equation}
for all agents $i \in [m]$, where $Q^{\pi}$ is the $Q$-function induced by following the policy $\pi$ in the future.

\section{Strategic relevance}  \label{sec:strategic_relevance}

\subsection{The minimal knowledge required to coordinate}

\begin{figure}[t]
    \centering
    \includegraphics[width=\linewidth]{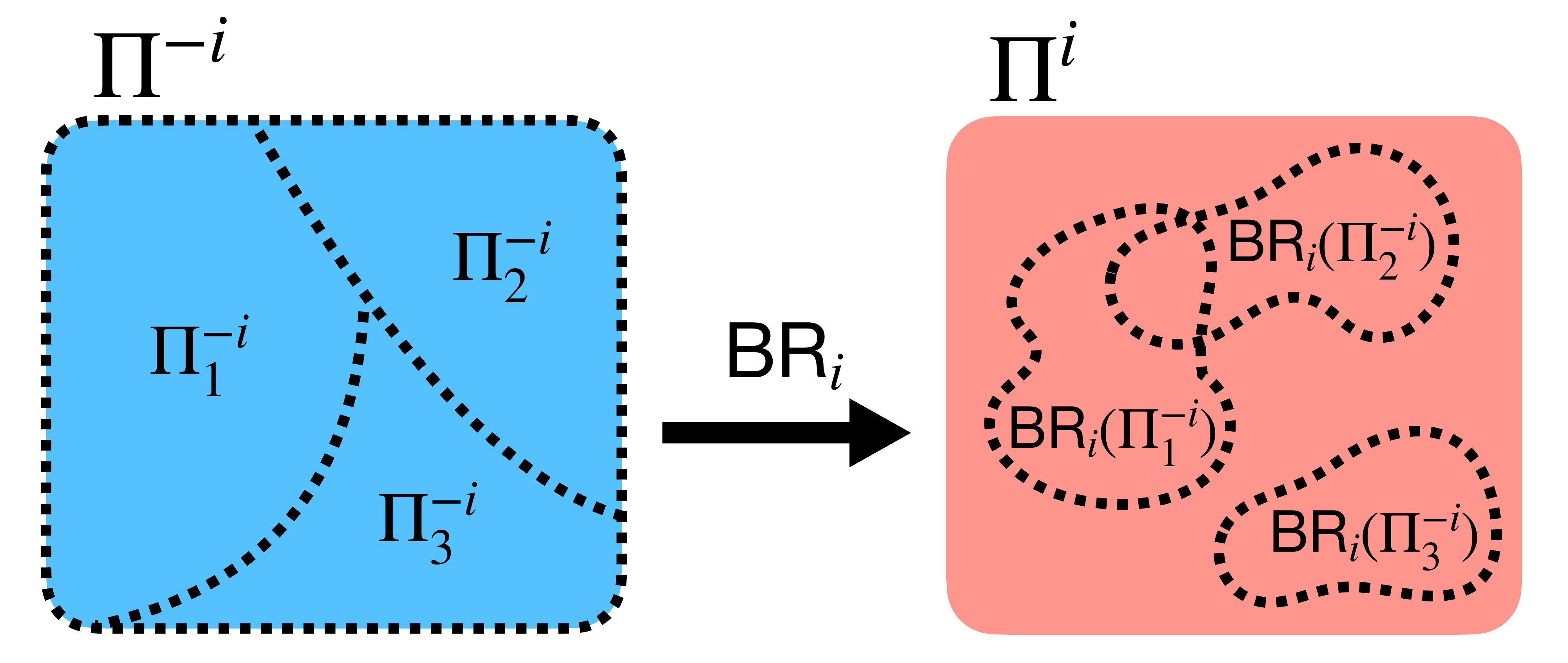}
    \caption{An abstract visualization of how the best-response function partitions the co-policy space into strategic equivalence classes. The best-response function maps the strategic equivalence classes from the co-policy space $\Pi^{-i}$ (left) to their best-response subset of the policy space $\Pi^{i}$ (right).}
    \label{fig:br_viz}
\end{figure}

Consider an agent, Alice, driving on the road. Alice feels safe driving on the road because she knows that the drivers around her will follow certain conventions and rules. She knows that other drivers will stop at red lights and go at green lights, that everyone will drive on the right side of the road, and that everyone will alternate who gets to go at stop signs. All of these conventions give Alice knowledge about other driver's policies that is relevant to how she can successfully coordinate with them on the open road. Despite this knowledge that Alice has about other drivers' behaviors, many aspects, if not most, are unknown to Alice. She does not know what the other drivers' destinations are, whether the other cars are manual or automatic, or the fine-grained motor controls of drivers. Luckily, from Alice's perspective, these other details about other drivers' behavior are irrelevant to her task of safely driving on the road. 

In the rest of this section, we formalize the question: how can we distinguish between the knowledge of other agents' policies that is relevant and irrelevant to the task?

\subsection{Strategic ambiguity}
\label{sec:strategic_ambiguity}

Let $\Pi'_{-i} \subset \Pi_{-i}$ be the subset of policies that agent $i$ thinks their co-policy could possibly be, given their current knowledge. If agent $i$ can narrow down $\Pi'_{-i}$ to a small enough subset, then agent $i$ has enough knowledge to exactly characterize what their optimal response should be. In such settings, we call $\Pi'_{-i}$ \textit{strategically unambiguous}.

\begin{definition}
  A nonempty subset $\tilde{\Pi}_{-i} \subset \Pi_{-i}$ of policies is \textit{strategically unambiguous} if the best-response set to all elements of $\tilde{\Pi}_{-i}$ are identical: $\br_i(\pi_{-i}) = \br_i(\pi'_{-i})$ for all $\pi_{-i}, \pi'_{-i} \in \tilde{\Pi}_{-i}$. We call sets of policies that don't satisfy this property \textit{strategically ambiguous}.
\end{definition}

\begin{figure*}[t]
    \centering
    \begin{minipage}{0.5\linewidth}
        \centering
        \setlength{\extrarowheight}{2pt}
        \begin{tabular}{cc|*{3}{W{c}{4mm}|}}
          & \multicolumn{1}{c}{} & \multicolumn{3}{c}{Player $Y$}\\
          & \multicolumn{1}{c}{} & \multicolumn{1}{c}{$C$}  & \multicolumn{1}{c}{$D$} & \multicolumn{1}{c}{$E$} 
          \\\cline{3-5}
          \multirow{2}*{Player $X$}  & $A$ & $1$ & $0$ & $1$ 
          \\\cline{3-5}
          & $B$ & $0$ & $1$ & $0$
          \\\cline{3-5}
        \end{tabular}
        \subcaption{A common-payoff game between two players.} 
        \label{tbl:multiple_NE}
    \end{minipage}\hfill
    \begin{minipage}{0.5\linewidth}
        \centering
        \includegraphics[width=0.35\linewidth]{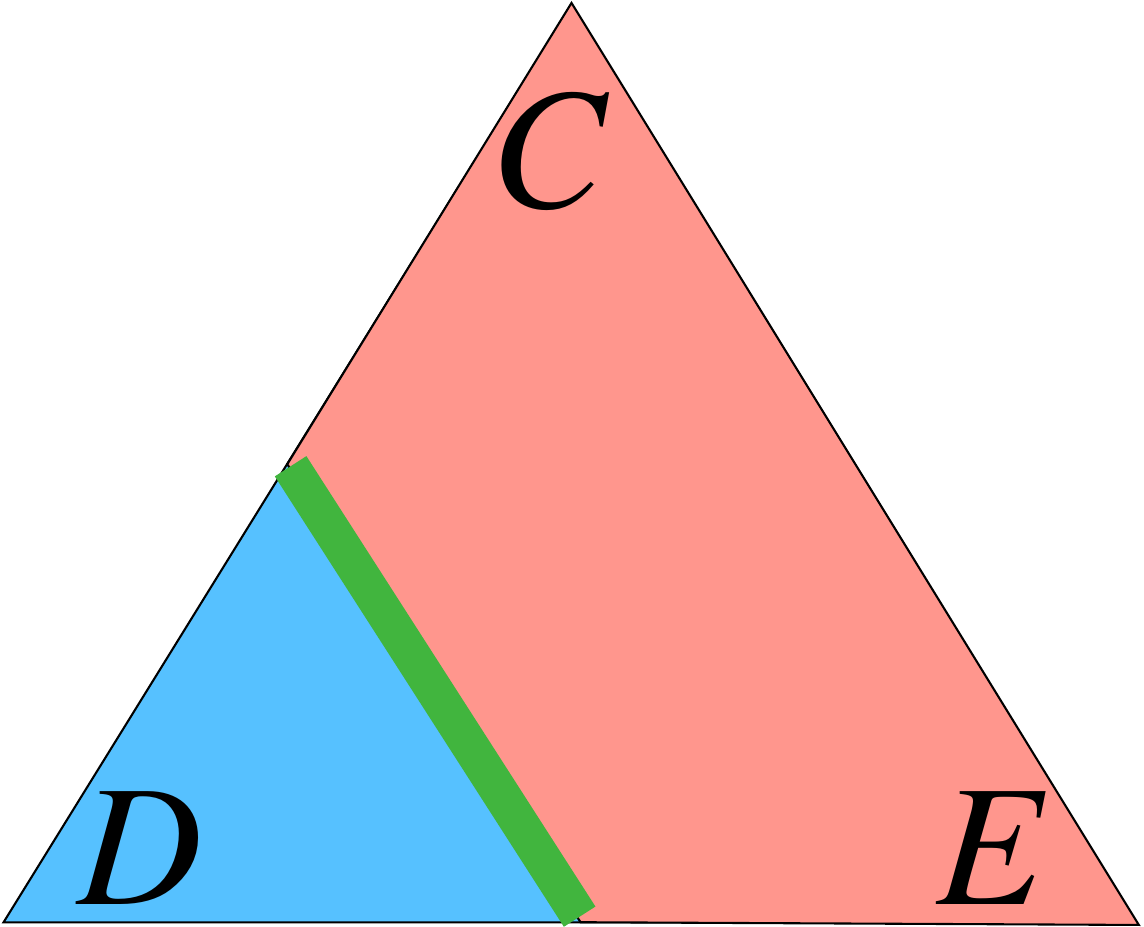}
        \subcaption{The strategic equivalence relation over player $Y$'s policy space. Vertices represent pure strategies.}
        \label{fig:strategic_partition_viz}
    \end{minipage}
    \caption{Figure \ref{tbl:multiple_NE} shows a common-payoff game between two players with multiple Nash equilibria. In Figure \ref{fig:strategic_partition_viz}, the red region (upper right) represents the policies for which $A$ is a best response, the blue region (lower left) represents the policies for which $B$ is a best response, and the green region (separating line, which technically has zero width) represents the policies for which both $A$ and $B$ are best responses.}
\end{figure*}

If $\Pi'_{-i}$ is strategically ambiguous, it means that agent $i$ does not have enough information to unambiguously know what their best response is; they are lacking some strategically relevant information. Let's consider the two extremes.
If agent $i$ knows agents $-i$'s policies exactly, then the policies $\Pi'_{-i}$ that agent $i$ thinks their co-players are following is a singleton and is always trivially strategically unambiguous.
On the other hand, if agent $i$ knows nothing, then $\Pi'_{-i} = \Pi_{-i}$ is the full policy space. However, in a fully decentralized problems where agent $i$ does not not need to know anything about their co-policy, the full policy space $\Pi_{-i}$ is still strategically unambiguous.

A natural next question to ask is: what is the \textit{minimal} amount of knowledge an agent needs to have about their co-policy to characterize what policies are optimal? Or dually, what is the \textit{maximal} subset of policies that is still strategically unambiguous?

\begin{definition}
A strategically unambiguous subset $\Pi'_{-i} \subset \Pi_{-i}$ of policies  \textit{contains no irrelevant information} if any strict superset of $\Pi'_{-i}$ is strategically ambiguous.
\end{definition}

A strict superset of $\Pi'_{-i}$ corresponds to a scenario in which agent $i$ has less knowledge. So a subset that contains no irrelevant information corresponds with a scenario in which \textit{any} less information would lead to strategic ambiguity, In other words, all of the strategically irrelevant information has already been discarded.

\subsection{Strategic equivalence relations}

Now we turn our attention to partitioning an agent's co-policy space in terms of best responses, and show how this relates to strategic ambiguity.

All of the strategically unambiguous subsets of $\Pi_{-i}$ that contain no irrelevant information are given by taking the preimage of the best response function for agent $i$.
\begin{theorem} \label{thm:br_partition}
  The strategically unambiguous subsets of $\Pi_{-i}$ that contain no irrelevant information are given by the preimage of the best-response map $\br_i: \Pi_{-i} \to \powerset{\Pi_i}$.
\end{theorem}
\vspace{-1em}
\begin{proof}
    We will show that a strategically unambiguous subsets of $\Pi_{-i}$ contain no irrelevant information if and only if it is in the preimage of the best-response map. Consider $\br_i^{-1}: \powerset{\Pi_{i}} \to \Pi_{-i}$, the inverse of $\br_i$.

    ($\impliedby$) Let $B = \br_i^{-1}(A)$ be the preimage of $A$, an arbitrary set in the domain of $\br_i$. By definition, for any co-policy $\pi_{-i} \in \Pi_{-i}$, we have $\br_i(\pi_{-i}) = A$ if and only if $\pi_{-i} \in B$. Therefore, any strict superset of $B$ must contain a policy that has a best-response set different from $A$, which would introduce strategic ambiguity. Therefore, $B$ contains no irrelevant information.

    ($\implies$) Now let $B$ be a strategically unambiguous subset of $\Pi_{-i}$ that contains no irrelevant information. Since $B$ is strategically unambiguous, by definition, all policies in $B$ must have the same best-response set; call this set $A$. Since $B$ contains no irrelevant information, it must contain \textit{all} policies in $\Pi_{-i}$ that induce best response $A$. Therefore, $B = \br_i^{-1}(A)$ and is in the preimage of $\br_i$.
\end{proof}
\vspace{-1em}
In this sense, the preimage of $\br_i$ gives us the coarsest (fewest element) partitioning of $\Pi_{-i}$ into subsets that are strategically unambiguous. We call the equivalence relation induced by this partitioning the \textit{strategic equivalence relation}.
\begin{definition} \label{def:strat_eq_rel}
  The \textit{strategic equivalence relation} (SER) $\sim_i$ for player $i$ is the equivalence relation over the co-policy space $\Pi_{-i}$ such that $\pi_{-i} \sim_i \pi'_{-i}$ if and only if $\br_i(\pi_{-i}) = \br_i(\pi'_{-i})$. 
  We refer to the equivalence classes of the SER as \textit{strategic equivalence classes} (SECs), which partition the co-policy space.
  We write $[[\pi_{-i}]]$ to denote the SEC that contains $\pi_{-i}$.
\end{definition}
Figure \ref{fig:br_viz} shows how the best-response
function partitions the co-policy space into SECs. The SECs are exactly the strategically unambiguous subsets of $\Pi_{-i}$ that contain no irrelevant information. If two policies fall in the SEC, then they must induce the same best response from agent $i$, by definition. Knowledge that allows agent $i$ to differentiate between policies within the same class is extraneous, since it has no affect on agent $i$'s choice of optimal policies. In this sense, which equivalence class contains $\pi_{-i}$ is the minimum knowledge that agent $i$ needs to compute their set of best responses.

Consider the payoffs in Table \ref{tbl:multiple_NE}.
Suppose player $X$ knows that player $Y$'s policy takes the form $w_1 C + w_2 E$.
Notice that although different weights $w_1,w_2$ give distinct policies, they all fall into the same SEC because they induce the same best response, $A$, from player $X$. All policies that are more likely to choose $D$ than $C$ or $E$ (combined) fall into a separate SEC and policies that choose $D$ and $C$ or $E$ (combined) with equal probability fall in a third. 
Figure \ref{fig:strategic_partition_viz} shows player $Y$'s mixed policy space and its three SECs. In order for player $X$ to play optimally, they only need to know which of the three SECs player $Y$'s policy falls in.

SERs partition policies into classes depending on the best response that they induce, regardless of the payoff that those policies induce. It is possible for two policies to be strategically equivalent $\pi_{-i} \sim_{i} \hat{\pi}_{-i}$ (i.e., they induce the same best response from agent $i$), while inducing different payoffs. A \textit{valued strategic equivalence relation} is a refinement of an SER that differentiates such policies. 
\begin{definition} \label{def:vser}
  A \textit{valued strategic equivalence relation} (VSER) $\stackrel{\text{v}}{\sim}_i$ for player $i$ is an equivalence relation over the co-policy space $\Pi_{-i}$ such that $\pi_{-i} \stackrel{\text{v}}{\sim}_i \hat{\pi}_{-i}$ if and only if,
  \vspace{-0.5em}
  \begin{enumerate}[label=(\roman*)]\setlength\itemsep{0em}
    \item $\br_i(\pi_{-i}) = \br_i(\hat{\pi}_{-i})$, and
    \item $R(\pi_i, \pi_{-i}) = R(\pi_i, \hat{\pi}_{-i})$ for all $\pi_i \in \br_i(\pi_{-i})$.
  \end{enumerate}
\end{definition}

\begin{figure*}[t]
     \centering
     \subfloat[A two-step dynamic game. The agents first take a joint action that transitions them into one of two final states, and then they receive a joint payoff according to their second joint action.]{\makebox[23.5em][c]{\includegraphics[scale=0.74]{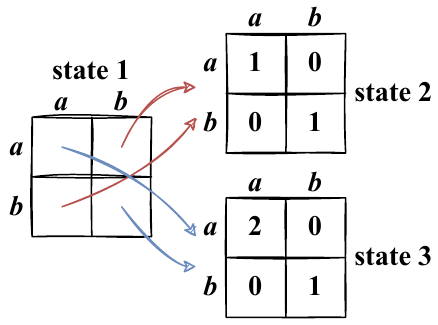}}\label{fig:simple_dynamic_game}} 
     \hfill
     \subfloat[The VSER DAG $\SECstructure_i$ for both players. Each box represents a node in $\SECstructure_i$ with the format ${[(A_i^*, V^*) : \hat\Pi_{-i}]}$ where $A_i^*$ is the best-response set, $V^*$ is the best-response value, and $\hat\Pi_{-i}$ are the one-step co-policies, the actions the co-policies in the equivalence class prescribe at state $s$.]{\makebox[23.5em][c]{\includegraphics[scale=1.0]{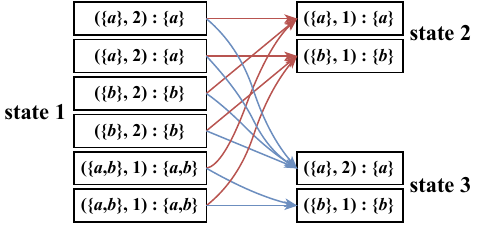}}\label{fig:example_dag}}
     \caption{A dynamic game along with the VSER DAG $\SECstructure_i$ for both players (since the game is symmetric).}
     \label{fig:steady_state}
\end{figure*}

\vspace{-0.5em}
\section{Computing strategic equivalence relations} \label{sec:compute_convs}

In the most general setting, our theory applies to strategic equivalence relation (SER) over the space of \textit{mixed} policies. Even though the number of SECs is still guaranteed to be finite in the case of mixed policies (see Lemma \ref{lem:sec_upperbound} in Appendix \ref{appendix:sec} for more details), we have a stronger bound in the case of pure policies, where the number of strategic equivalence classes is upper-bounded by the number of pure policies (see Appendix \ref{appendix:mixed_policies} for more details). Moreover, it is computationally easier to only compute the set of best-response policies over pure strategies in DecPOMDPs \cite{szer2006pointbased}. For these reasons, we limit ourselves to computing SERs over pure strategies. For similar computational reasons, we restrict ourselves to the domain of common-payoff games, although all of our theory extends to general-sum games as well.

In Section \ref{sec:computing_strategic_equivelence_relations} we compute the VSER over the set of best-response policies in DecPOMDPs. 
In Appendix \ref{sec:compute_info}, we give a simplified (and faster) algorithm for computing VSERs over Nash equilibrium in fully-observed DecPOMDPs. For more details and rationale behind computing strategic equivalence relations over subsets of the full policy space, see Appendix \ref{appendix:other_subsets}.

\subsection{Normal-form games} \label{sec:nf_games}

As a warmup, we quickly explain how to compute the strategic equivalence relation (Definition \ref{def:strat_eq_rel}) of a normal-form game (i.e., a DecPOMDP with a single state).
In order to compute the strategic equivalence relation for agent $i$, we just need to enumerate the pure policies $\pi_{-i} \in \Pi_{-i}$, compute the best responses $\br_i(\pi_{-i}) = \argmax_{\pi_i \in \Pi_i} R(\pi_i, \pi_{-i})$ for each one, and group the policies that induce the same best responses.

\subsection{Strategic equivalence graphs} \label{sec:se_graphs}

In DecPOMDPs with multiple states, instead of explicitly computing the policy within every valued strategic equivalence class (VSEC) (see Definition \ref{def:vser}), we compute a more compact structure that takes advantage of the inherently recursive nature of optimal policies. This will entail computing a directed acyclic graph $\SECstructure_i$ as follows for each agent $i$.

For simplicity, let's first restrict ourselves to the case of a \textit{fully-observed} DecPOMDP. 
For each subgame starting in state $s$ there exist multiple nodes in $\SECstructure_i$. Each node conceptually stores a single class of the \textit{intermediate VSER}, the VSER over policies restricted to the subgame starting at state $s$. Each node explicitly contains the,
\vspace{-1em}
\begin{enumerate}\setlength\itemsep{-.2em}
    \item The state $s$ that it's associated with;
    \item The \emph{one-step co-policies}: the actions that the co-policies in the SEC prescribe at state $s$;
    \item The set of best-responses $A_i^*$ for agent $i$ in state $s$;
    \item The value $V^*$ that they induce from agent $i$ in state $s$;
    \item The edges to children nodes associated with the successor states of $s$ that continue the VSEC in the future.
\end{enumerate}
Item (5) is required because the value of agent $i$'s actions at state $s$ are dependent on what the policies of all agents prescribe to future states. Each of these children nodes contains the information on how the VSEC continues in the future.

Each VSEC over the complete policy space is given by a root node of $\SECstructure_i$ (which are each associated with the starting state of the DecPOMDP) along with all of its children. 
For simplicity and to give computational speedup, nodes that share the same history, best-response actions, and best-response value can be merged, turning $\SECstructure_i$ into a multigraph (i.e. a graph that is allowed to have multiple edges between the same two nodes).

Figure \ref{fig:example_dag} shows $\SECstructure_i$ for the simple two-step game depicted in Figure \ref{fig:simple_dynamic_game}. Each box represents a node in $\SECstructure_i$ along with it's information in the format $[(A_i^*, V^*) : \hat\Pi_{-i}]$, where $\hat\Pi_{-i}$ are the one-step co-policies. 
See Appendix \ref{appendix:seg} for a detailed explanation of $\SECstructure_i$ in this example.

The primary change to $\SECstructure_i$ in partially-observed DecPOMDPs is that nodes must be associated with histories, not states, since optimal policies are computed over history-dependent beliefs. 
Let $\Pi^{h_i}_{-i} \subset \Pi_{-i}$ denote the subset of co-policies that are \textit{consistent} with history $h_i$, i.e., there exists some random outcomes that make $h_i$ a valid history. The nodes associated with history $h_i$ in $\SECstructure_i$ store intermediate VSERs $\stackrel{\text{v}}{\sim}_{h^t_i}$ over the subsets $\Pi^{h_i}_t$ based on the best response they induce from agent $i$, given that they've already observed $h_i$. Formally, for $\pi_{-i}, \pi'_{-i} \in \Pi^{h_i}_{-i}$, we have $\pi_{-i} \stackrel{\text{v}}{\sim}_{h^t_i} \pi'_{-i}$ if, 
\begin{align}
\begin{split}
    \br_{h_i}(\pi_{-i}) &= \br_{h_i}(\pi'_{-i}), \text{and} \\
    R(\pi_i, \pi_{-i}) &= R(\pi_i, \pi'_{-i}), \forall \pi_i \in \br_{h_i}(\pi_{-i}),
\end{split}
\end{align}
where $\br_{h_i}: \Pi^{h_i}_{-i} \to \Pi^{t}_i$ maps to the set of best-responses for agent $i$ from timepoint $t$ onward, given history $h_i$.

Each node in $\SECstructure_i$  conceptually stores a single class of the intermediate VSER $\stackrel{\text{v}}{\sim}_{h^t_i}$. However, each node does not specify the full policies in that class. Each node contains the one-step co-policies, the partial policies for agents $-i$ that prescribe actions to the roots of the policy that has already been played. The nodes also store the corresponding one-step best response actions $A_i^*$ and value $V^*$ that they induce from agent $i$ given history $h_i$. Like before, the value of these one-step actions is dependent on the actions taken by the policies at future timesteps. Therefore, nodes have one edge for every action-observation pair $(a_i, o_i) \in \actions_i \times \obss_i$ to other nodes associated with histories $h_i + (a_i, o_i)$. Each of these nodes contains information on how to continue the VSEC in the future.

\subsection{DecPOMDPs}
\label{sec:computing_strategic_equivelence_relations}

In this section we present an algorithm for simultaneously computing the set of best-response policies in a DecPOMDP along with the VSER over these policies.  The algorithm is based on point-based, dynamic programming approaches for computing the set of best-response policies in DecPOMDPs \cite{szer2006pointbased}.

At a high-level, Algorithm \ref{alg:decpomdp_main} is based on a dynamic programming operator (the function \textbf{BackupSER}) that computes best-response policies starting at timestep $t$, along with the intermediate VSECs over them, given the VSECs for policies starting at timestep $t+1$. Using this dynamic programming operator, we can then efficiently compute the full set of best-response policies along with their strategic equivalence classes through backwards induction over a finite-horizon DecPOMDP. Repeated application of this operator corresponds with iteratively computing the nodes of the VSER DAG $\SECstructure_i$ from the last timestep to the first.

As hinted by the structure of $\SECstructure_i$ in Section \ref{sec:se_graphs}, intermediate VSECs at history $h^t_i$ can be recursively computed in terms of intermediate VSECs from histories that are one-step extensions of $h^t_i$. The following theorem captures this idea.

\begin{theorem}[Bellman backup for strategic equivalence] \label{thm:bellman}
   For all histories $h^t_i$ and policies $\pi_{-i}, \pi'_{-i} \in \Pi^{h^t_i}_{-i}$, we have $\pi_{-i} \stackrel{\text{v}}{\sim}_{h^t_i} \pi'_{-i}$ if and only if,
  \vspace{-.5em}
  \begin{enumerate}[label=(\roman*)]\setlength\itemsep{0em}
    \item \label{thm:bellman:cond1} $\underset{a_i}{\argmax}\ Q_i^*(\pi_{-i}, h^t_i, a_i) = \underset{a_i}{\argmax}\ Q_i^*(\pi_{-i}',h^t_i, a_i)$, and $\underset{a_i}{\max}\ Q_i^*(\pi_{-i}, h^t_i, a_i) = \underset{a_i}{\max}\ Q_i^*(\pi_{-i}',h^t_i, a_i)$;
    \item \label{thm:bellman:cond2} $\pi_{-i} \stackrel{\text{v}}{\sim}_{h^t_i + (a_i,o_i)}\pi'_{-i}$ for all actions $a_i \in \argmax_{a_i} Q_i^*(\pi_{-i},h^t_i, a_i)$ and corresponding possible observations $o_i$.
  \end{enumerate}
\end{theorem}
\vspace{-1em}
\begin{proof}
 See Appendix \ref{appendix:misc_lemmas}
\end{proof}
\vspace{-.5em}

\begin{algorithm}[t]
    \caption{Computes the valued strategic equivalence relation over best-response policies in the form of $\SECstructure_i$.
    \label{alg:decpomdp_main}}
    \begin{algorithmic}[1]
        \INPUT{DecPOMDP $\mathcal{M}$.}

        \FUNCTION{\textbf{Main}}
        \SHORTFORALL{$t \in \{T, T-1, \dots, 0\}$}{\textbf{BackupSER}(t)}
        \ENDFUNCTION
        \FUNCTION{\textbf{BackupSER}($t$)}
            \FOR{all previous joint policies $\pi^{t}$}
                \FOR{each agent $i$ and history $h_i^{t}$}
                    \STATE $B (A_i^*, V^*) \gets \emptyset$
                    \FOR{$[[\pi_{-i}]] \in \mathscr{O}(\SECstructure, h_i^{t}, \pi^{t})$}
                        \STATE \textbf{Bellman}($Q^*_i, \pi_{-i}, h^t_i$)
                        \STATE $A_i^* \gets {\argmax}_{a_i}Q^*_i(\pi_{-i},h^t_i, a_i)$
                        \STATE $V^* \gets {\max}_{a_i}Q^*_i(\pi_{-i},h^t_i, a_i)$
                        \STATE $B (A_i^*, V^*).\text{add}( [[\pi_{-i}]])$
                    \ENDFOR
                    \SHORTFORALL{$(A_i^*, V^*)$}{ $\SECstructure_i.\text{addNode}(B  (A_i^*, V^*))$}
                \ENDFOR
            \ENDFOR
        \ENDFUNCTION
        \OUTPUT $\{\SECstructure_i\}_{i \in [m]}$
    \end{algorithmic}
\end{algorithm}

\begin{figure*}
     \centering
     \subfloat[\textbf{Overcooked -- ``Locked In'':} cooks need to coordinate on the timing of a recipe using a single pot.]{\includegraphics[width=.32\linewidth]{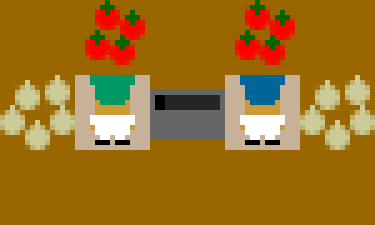}\label{fig:overcooked_tiny}} 
     \hfill
     \subfloat[\textbf{Overcooked -- ``Schelling'':} cooks need to coordinate on who gets to occupy the useful central tile without bumping into each other to deposit two onions in time.]{\makebox[15em][c]{\includegraphics[width=.28\linewidth]{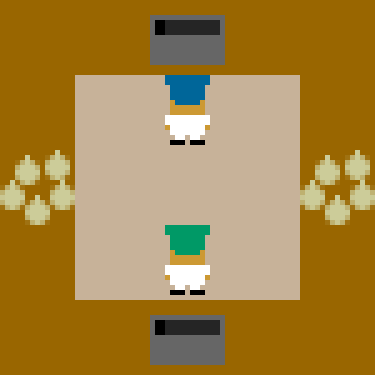}}\label{fig:overcooked_schelling}}
     \hfill
     \subfloat[\textbf{Overcooked -- ``Coordination Ring'':} cooks cannot occupy the same spot, requiring them to coordinate on how they pass around the central island.]{\makebox[15em][c]{\includegraphics[width=.28\linewidth]{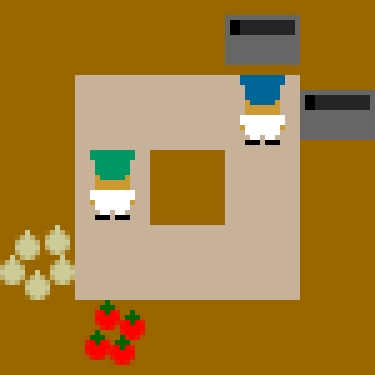}}\label{fig:overcooked_ring}}
     \caption{The three Overcooked environments analyzed in Section \ref{sec:experiments}, each requiring different types of coordination.}
    \vskip -0.03em
\end{figure*}

As we see below, if we back up the optimal $Q$-functions, $Q^*_{\pi_{-i}}$ based on future timesteps, we can check condition \ref{thm:bellman:cond1}. Since we have already computed VSECs for all future histories, we can check condition \ref{thm:bellman:cond2}.  

The optimal $Q$-functions $Q^*_{i}$ can be backed up using a recursive computation analogous to the classic Bellman backup operator (see Lemma \ref{lem:bellman_optim} in Appendix \ref{appendix:misc_lemmas}). This operation is defined as the function \textbf{Bellman} on line 9 of Algorithm \ref{alg:decpomdp_main}. Now we give a detailed explanation of the remaining components of Algorithm \ref{alg:decpomdp_main}.

Lines 5 and 6 iterate over each past joint policy $\pi^t$, agent $i$, and possible history $h^t_i$ under $\pi^t$ for agent $i$. In deterministic environments, we can reduce line 5 to iterating over policy chains (i.e., sequences of actions), rather than policy trees, since trajectories are deterministic given fixed sequences of actions. All algorithms (including baselines) in our experimental section make use of this optimization. Check Appendix \ref{append:alg1_remarks} for more details.

The operator $\mathscr{O}(\SECstructure_i, h_i^{t}, \pi^{t})$ on line 8 conceptually iterates over the class $[[\pi_{-i}]]$ of policies that are equivalent under condition \ref{thm:bellman:cond2}. 
It is implemented by (1) enumerating the one-step co-policies, i.e., assignments of actions to the leaves of $\pi_{-i}^t$ and (2) enumerating over choices of nodes in $\SECstructure_i$ associated with all possible one-step extensions $h_i^t + (a_i, o_i)$ of $h_i^t$.

Consider computing the VSECs (Figure \ref{fig:example_dag}) in state 1 of the game in Figure \ref{fig:simple_dynamic_game} given the intermediate VSECs of states 2 and 3. The operator $\mathscr{O}$ enumerates all choices of nodes for the two future states, and all assignments to the co-policy at state 1. For example, choosing the node \fbox{$(\{a\}, 1) : \{a\}$} in state 2 and node \fbox{$(\{b\}, 1) : \{b\}$} in state 3 and the one-step co-policy ${a}$ produces best response $A_i^* = \{a,b\}$ in state 1.

Line 10 and 11 record the optimal actions and their value. Line 12 partitions the sets of co-policies based on their valued best response, i.e. checking condition \ref{thm:bellman:cond1}. Finally, line 14 takes the computed classes and adds the corresponding node and edges to  $\SECstructure_i$.

\section{Experiments}
\label{sec:experiments}

The purpose of our experiments is twofold: (1) to characterize and gain intuition about what the strategic equivalence relation (and its equivalence classes) look like in practice, and (2) to empirically evaluate the scaling performance of our algorithms. We compute the strategic equivalence relation (SER) for different environments and provide several videos made available \href{https://minknowledge.github.io}{on our website} that include 1-4 sampled representative policies from each class. 

Our primary evaluation environment is the Overcooked environment \cite{micah2019human}, in which players control chefs that cook meals in a kitchen. Agents need to coordinate on the high-level strategy for collecting the different ingredients in a dish and the low-level motion controls to avoid getting in each other's way. The onion and tomato tiles hold an infinite supply of the ingredients and agents get positive reward for placing various combinations of vegetables in the pots, depending on the specific environment. Agents have access to different movement actions (depending on the environment) and an interact action which picks up or places ingredients depending on what tile the agent is facing.
We provide analysis and visualization of additional Overcooked environments in Appendix \ref{append:extra_overcooked}.

\subsection{Fully-observed settings}

\textbf{Overcooked -- Locked-in.} The first Overcooked environment we investigate is visualized in Figure \ref{fig:overcooked_tiny}. We let the agents' action space be composed of the following: rotate left, rotate right, and interact. The agents need to collaborate to make a soup containing an onion and a tomato within 7 timesteps. However, the recipe requires placing the onion into the pot precisely one timestep after the tomato to get +1 reward, or the soup will be ruined.

This environment has two strategic equivalence classes: one class has the left agent collecting an onion and the right agent collecting a tomato, while the other class has the roles switched. Because the environment is fully-observed, both classes have both agents delaying their commitment to either vegetable at the beginning of the episode to see which vegetable their co-player will choose, and then at some point breaking the symmetry by committing to a vegetable. See the two videos \href{https://minknowledge.github.io/locked#full-information}{on our website} for visualizations.

In the rest of this section, we investigate computing valued strategic equivalence relations (VSECs) over trembling-hand subgame perfect Nash equilibrium using the optimized algorithm described in Appendix \ref{sec:compute_info}.

\textbf{Overcooked -- Schelling Point.} This Overcooked environment is depicted in Figure \ref{fig:overcooked_schelling}. Agents get a reward of +1 for each onion that is put into either of the pots. We let agents also have a larger action space: move left, right, up, down, stay still, and interact. We use a horizon of $H=8$. In this environment agents need to coordinate on who gets to occupy the useful central tile without running into each other (which results in a no-op). Optimal policies achieve a total reward of +2.

This environment has only two equivalence classes, depending on which agent gets to occupy the central square first. After one agent occupies the central square, the symmetry in the problem is broken and the second agent's best responses become fixed. See the two videos \href{https://minknowledge.github.io/schelling}{on our website} for visualizations.

\textbf{Overcooked -- Coordination Ring.} This Overcooked environment is depicted in Figure \ref{fig:overcooked_ring}. Agents get a reward of +1 as soon as there is an onion and tomato in the same pot. Again we let agents also have a larger action space: move left, right, up, down, stay still, and interact. We use a horizon of $H=11$, although different values of $H$ produce similar results. 

At first examination, one might predict that there should be multiple VSEC in this problem, perhaps resulting from which agent gets which vegetable, which pot to use, or the various options for how the agents could walk around the central island to stay out of each other's way. However, all of these variations are found within the same strategic equivalence class!

Surprisingly, this environment has only a \textit{single} VSEC for either agent. Conceptually, this means that both agents \textit{do not need to know anything} about their co-policy in order to play optimally; agents require no preexisting agreements on how to successfully coordinate in the environment. Since the environment is fully-observed, the blue agent is able to watch the green agent's actions and best respond no matter what strategy green chooses. See the video \href{https://minknowledge.github.io/ring}{on our website} for a visualization of this VSEC.

\begin{figure}
    \centering
    \includegraphics[width=.65\linewidth]{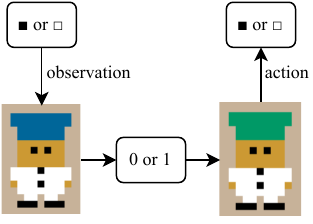}
    \caption{\textbf{Referential game:} Blue (left) first randomly observes either $\blacksquare$ or $\square$. Green (right) observes a 0/1 signal from Blue and chooses $\blacksquare$ or $\square$, trying to match Blue's observation.}
    \label{fig:referential_game}
\end{figure}

\begin{figure*}
     \centering
     \subfloat[\textit{Fully-observed} (Algorithm in Appendix \ref{sec:compute_info}) uses the fully-observed version of the environment. \textit{Partially-observed} (Algorithm \ref{alg:decpomdp_main}) and \textit{Enumerative baseline} (described in Section \ref{sec:benchmark_decpomdp_planning}) use the partially-observed version of the environment.]{\begin{tikzpicture}

\definecolor{darkgray176}{RGB}{176,176,176}
\definecolor{darkorange25512714}{RGB}{255,127,14}
\definecolor{forestgreen4416044}{RGB}{44,160,44}
\definecolor{steelblue31119180}{RGB}{31,119,180}

\begin{axis}[
mark size=3,
height=5cm,
width=8cm,
legend cell align={left},
legend style={fill opacity=0.8, draw opacity=1, text opacity=1, at={(0.98,0.02)}, anchor=south east, draw=white!80!black},
log basis y={10},
tick align=outside,
tick pos=left,
x grid style={black},
xlabel={Horizon},
xmin=0.45, xmax=12.55,
xtick style={color=black},
y grid style={black},
ylabel={Time (seconds)},
ymin=1.3435938144627e-06, ymax=1894.55332712823,
ymode=log,
ytick style={color=black},
ytick={1e-08,1e-06,0.0001,0.01,1,100,10000,1000000},
yticklabels={
  \(\displaystyle {10^{-8}}\),
  \(\displaystyle {10^{-6}}\),
  \(\displaystyle {10^{-4}}\),
  \(\displaystyle {10^{-2}}\),
  \(\displaystyle {10^{0}}\),
  \(\displaystyle {10^{2}}\),
  \(\displaystyle {10^{4}}\),
  \(\displaystyle {10^{6}}\)
}
]
\addplot [very thick, forestgreen4416044, mark=|]
table {%
1 0.000292629999999905
2 0.00148614299999994
3 0.0390521719999999
4 0.96808691
5 28.469508398
6 727.163952321
};
\addlegendentry{Enumerative baseline}
\addplot [very thick, steelblue31119180, mark=|]
table {%
1 3.50060000000596e-06
2 0.000451049200000009
3 0.0052005956
4 0.0725061821
5 0.8929676979
6 9.9627010962
7 115.5845832549
};
\addlegendentry{Partially-observed}
\addplot [very thick, darkorange25512714, mark=|]
table {%
1 3.95200000002927e-06
2 0.00140432599999996
3 0.010815425
4 0.056928686
5 0.189656735
6 0.469652239
7 0.895915125
8 1.736391487
9 3.123782996
10 5.584874926
11 9.312807397
12 16.335166951
};
\addlegendentry{Fully-observed}

\end{axis}

\end{tikzpicture}\label{fig:benchmark}} 
     \hfill
     \subfloat[The time it take to compute the set of best-response policies using our algorithm (Algorithm \ref{alg:decpomdp_main}) and the baseline from \citet{szer2006pointbased}. Beyond a horizon of two, the baseline timed-out at 30 minutes.]{\begin{tikzpicture}

\definecolor{color0}{rgb}{0.12156862745098,0.466666666666667,0.705882352941177}
\definecolor{color1}{RGB}{214,39,40}

\begin{axis}[
mark size=3,
height=5cm,
width=8cm,
log basis y={10},
tick align=outside,
tick pos=left,
legend cell align={left},
legend style={fill opacity=0.8, draw opacity=1, text opacity=1, draw=white!80!black, at={(0.98,0.02)},anchor=south east},
x grid style={white!69.0196078431373!black},
xlabel={Horizon},
xmin=0.4875, xmax=6.2625,
xtick style={color=black},
y grid style={white!69.0196078431373!black},
ylabel={Time (seconds)},
ymin=4.12709008193735e-06, ymax=1214.62879500865,
ymode=log,
ytick style={color=black}
]
\addlegendentry{Our algorithm}
\addplot [very thick, color0, mark=|]
table {%
1 1.00135803222656e-05
2 0.000508308410644531
3 0.00631284713745117
4 0.112706184387207
5 0.977018117904663
6 11.5587148666382
};
\addlegendentry{Baseline}
\addplot [very thick, color1, mark=|]
table {%
1 7.70092010498047e-05
2 0.00056004524230957
3 500.608402967453
};
\end{axis}

\end{tikzpicture}\label{fig:decpomdp_benchmark}}
     \caption{The time (on a log scale) for computing strategic equivalence relations over policies of varying horizon using varying algorithms in the Overcooked environment depicted in Figure \ref{fig:overcooked_tiny}. Note that all algorithms make use of the optimization when enumerating joint policies in deterministic environments (see Appendix \ref{append:alg1_remarks}).} 
\end{figure*}
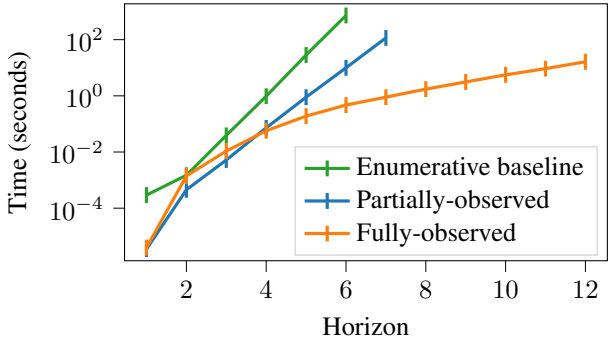
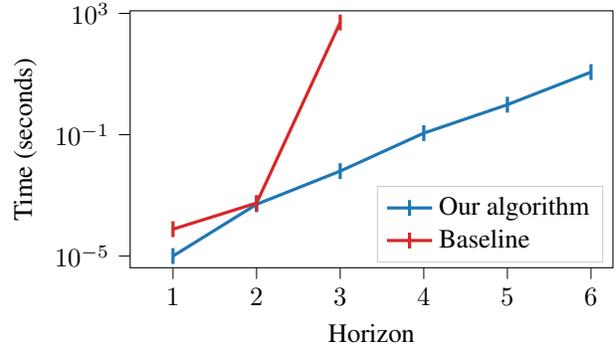

\subsection{Partially-observed environments}

\textbf{Referential Game.} The first partially-observed environment we investigate is a referential game between two agents described in Figure \ref{fig:referential_game}. The tricky aspect of this game is that both agents need to agree on the meaning of the signal that is communicated across the channel.

We compute three natural VSECs in this game: one class where the agents agree that $\blacksquare \mapsto 0$ and $\square \mapsto 1$ (which achieves optimal reward of 1), another class where the meaning is swapped, and a third class where the green agent (right) ignores the meaning of the signal and tries to guess the blue agent's (left) observation with expected value $\frac{1}{2}$.

\textbf{Overcooked -- Locked-in.} The second partially-observed environment we investigate is a modified version of the Overcooked environment depicted in Figure \ref{fig:overcooked_tiny}. In this version of the environment, the central pot blocks the agents' vision, preventing them from observing the position or actions of their co-player. This modification greatly increases the degree of coordination required between the agents because they can not rely on observing and adapting to their co-player's behavior. The reward is the same as before: agents get +1 reward whenever an onion is put into the pot precisely one timestep after a tomato is put into the pot.

This greater degree of coordination is clearly reflected in the size of the strategic equivalence relation: there are six classes for either agent that vary along two axis: (1) who takes charge of which vegetable and (2) whether the tomato should enter the pot on timestep 4, 5, or, 6 (and onion one timestep after). See the six videos \href{https://minknowledge.github.io/locked#partial-information}{on our website} for visualization.

\subsection{Computational efficiency}
\label{sec:benchmark_decpomdp_planning}

Since strategic equivalence relations are a novel concept, there are no pre-existing algorithms for computing them. We compare our algorithms against an enumerative baseline. The enumerative baseline computes strategic equivalence relations by first flattening the (partially observed) DecPOMDP down to a normal-form game where actions in the normal-form game represent policies in the DecPOMDP. Then, strategic equivalence classes are computed by enumerating co-policies and grouping by best response.

We experiment on increasing the horizon of policies in the Overcooked environment depicted in \Cref{fig:overcooked_tiny} and testing how long it takes to compute the strategic equivalence relation using different algorithms. From \Cref{fig:benchmark} (which is on a log scale), we can see the clear benefit of exploiting recursive structure to compute strategic equivalence classes, with results amplified using the optimized algorithm for a fully-observed version of the same setting. \Cref{fig:benchmark_nonlog} in the Appendix shows the variant of this plot on an absolute scale.

Since our algorithm also computes the set of best-response policies, it provides notable computational complexity improvements over existing algorithms \cite{szer2006pointbased, seuken2007memory} for computing the best-response policies in DecPOMDPs. By grouping co-policies into VSECs, the backup step of our algorithm avoids redundant computation from computing the best-responses to co-policies that are in the same SEC. Therefore, the number of multiagent beliefs that we need to consider for each history no longer scales in the exponential space of our co-policies, but rather in the number of intermediate VSECs over those co-policies. Figure \ref{fig:decpomdp_benchmark} shows the time it take to compute the set of best-response policies using our algorithm (Algorithm \ref{alg:decpomdp_main} and \textit{Partially-observed} in Figure \ref{fig:benchmark}) and the baseline from \citet{szer2006pointbased}. All algorithms are implemented in Python using the same data structures and use the optimization in deterministic environments (see Appendix \ref{append:alg1_remarks}).

\section{Discussion} \label{sec:discussion}
We have shown that the distinction between strategically relevant and irrelevant information can be formalized via the novel concept of strategic equivalence relation (SER). By providing a compact representation of this relation, and an efficient Bellman backup to compute it, we have provided an approach to efficiently understand what is strategically relevant to any given task. This allows us to shed new light on existing coordination benchmarks. For instance, we show in Section \ref{sec:experiments} that there is non-trivial strategically relevant information in Overcooked only in situations in which simultaneous decisions between incompatible optimal joint plans must be made. Given that such situations are rare, this provides a theoretical explanation as to why relatively good coordination with humans is achievable in this domain without any human data~\cite{strouse2021collaborating}. Our framework also explains why the introduction of simultaneous decisions (most easily through partial observability) can lead to benchmarks that are specifically more challenging for coordination, such as Hanabi~\cite{bard2020hanabi}.

The SER could serve as a critical component of both centralized and decentralized approaches for creating cooperative policies. In a decentralized settings, the SER tells us exactly what is needed in order to optimally coordinate with co-players. This could allow a focused effort on few-shot or zero-shot adaptation to co-players \cite{Zand2022OntheflySA, albrecht2017reasoning, stone2010adhoc, hu2021otherplay}, 
or tell us how to prioritize strategically relevant information if only a limited communication bandwidth is available \cite{pmlr-v119-wang20i, Mao_Zhang_Xiao_Gong_Ni_2020, berna2004communic}.

In a centralized setting, the SER tells us what must be agreed upon before the problem may be reduced to independent single-agent problems. As such, it can provide critical information for finding useful decompositions of centralized value functions \cite{jin2022v, cassano2021decentral, Wang2019LearningND}, or more efficient DecPOMDP planning algorithms \cite{szer2006pointbased, seuken2007memory}. In fact, in Section \ref{sec:benchmark_decpomdp_planning}, we showed that our algorithm for computing the SER already represents a complexity improvement in the state-of-the-art for computing the set of best-response strategies. 

While we are conscious of the computational challenges to scale our approach to complex domains, we are optimistic about future work to approximate our method. In this work, we aim to provide the theoretical foundations for this effort. Ultimately, using SERs to both better understand the challenges of cooperative tasks and to accelerate our algorithms for solving cooperative problems are both exciting directions for future research.

\section{Acknowledgement}

We thank Andrew Critch, Marcell Vazquez-Chanlatte, Jakob Foerster, Cassidy Laidlaw, Anand Siththaranjan, and Alyssa Li Dayan for helpful discussions at various stages of this project.
This research was supported by a gift from the Open Philanthropy Foundation to the Center for Human-Compatible AI and the Schmidt AI 2050 Fellowship (Russell). Niklas Lauffer and Micah Carroll are supported by a National Science Foundation Graduate Research Fellowship. Ameesh Shah is supported by an NDSEG Fellowship.

\clearpage

\bibliography{bibl}
\bibliographystyle{icml2023}

\newpage
\onecolumn
\appendix

\section{Strategic relevance} \label{appendix:sec}

\begin{remark}
 The definition related to strategic ambiguity and strategic equivalence relations hold for any response function $\Pi_{-i} \to \Pi_{i}$, not just the best response function defined in the main text, which assumes the agents' policies to be perfectly rational. Some appealing alternatives might include models of irrationality, such as ``Boltzmann" best response or $\epsilon$-best response. Exploring strategic equivalence relations in these contexts is left for future work.
\end{remark}

\setcounter{theorem}{1}

\subsection{Additional Lemmas}

We can extend the notion of strategic equivalence to capture all perspectives of a team of agents. The preimage of the joint best-response function $\br: \Pi \to \powerset{\Pi}$ forms an equivalence relation $\sim$ over $\Pi$. In turn, $\sim$ partitions $\Pi$ into equivalence classes that capture the minimum sufficient knowledge required to characterize the best response set for \textit{every} agent.
A natural question to ask is what (if any) the connection between $\sim$ and $\sim_i$ is. The following Lemma answers this question.
\begin{lemma}
  $\pi \sim \pi'$ if and only if $\pi_{-i} \sim_{i} \hat{\pi}_{-i}$.
\end{lemma}
\begin{proof}
  $\pi \sim \pi'$ if and only if $\br(\pi) = \br(\pi)$ if and only if $\br_i(\pi_{-i}) = \br_i(\pi_{-i}'), \forall i \in [m]$ if and only if $\pi \sim_i \pi', \forall i \in [m]$.
\end{proof}

Another result that follows from our definitions is that the number of strategic equivalence classes (even in the case of mixed policies) is always finite whereas the space of policies or Nash equilibria can be infinite.
\begin{lemma} \label{lem:sec_upperbound}
  The number of strategic equivalence classes is finite, upper bounded by $\prod_{i} 2^{|A_i|}$.
\end{lemma}
\begin{proof}
  Each best response set is the convex combination of a subset of pure strategies. There are $\prod_{i} 2^{|A_i|}$ such unique subsets of pure strategies.
\end{proof}

\section{Computing strategic equivalence relations}

\subsection{Mixed policies can have an exponential number of strategic equivalence classes} \label{appendix:mixed_policies}

 Consider two agents in a DecPOMDP with a single state and identical finite strategy sets $\actions$. The agents receive a payoff of $1$ if they play the same strategy and $0$ if they differ. In this game, depending on what the co-player does, it is possible for every single subset of $\actions$ to be a strategic equivalence class. Namely, for each subset $A \subset \actions$, the simplex over $A$ is the best response to the uniform policy over $A$. Moreover, since $\pi$ is a uniform policy, the simplex over $S$ defines the set of best responses to this policy. Therefore, there are an exponential (in the size of $A$) number of distinct strategic equivalence classes, one for every subset of $A$. 

 \subsection{Example strategic equivalence graph} \label{appendix:seg}

Figure \ref{fig:example_dag} shows the DAG $\SECstructure_i$ for the simple two-step game depicted in Figure \ref{fig:simple_dynamic_game}.
Each box represents a node in $\SECstructure_i$ along with it's information in the format $(A_i^*, V^*) : \hat\Pi_{-i}$ where $\hat\Pi_{-i}$ are the one-step co-policies. The two boxes in the top-right of Figure \ref{fig:example_dag} represent the two nodes in $\SECstructure_i$ associated with the subgame beginning in the top-right state in Figure \ref{fig:simple_dynamic_game}. The first (second) node show that action $a$ ($b$) is the best-response set to their co-player choosing action $a$ ($b$) and achieves value 1. The six boxes on the left of Figure \ref{fig:example_dag} represent the root nodes of the six SECs of the full game. Each box shows the best-response set, the value of the best-response, and the actions of the policies in that class. The red and blue arrows point to future states in the game and how that class is continued at those states. Notice the pairs of nodes in the first layer that can be merged, because they share the same valued best response.
 
 \subsection{Miscellaneous Lemmas} \label{appendix:misc_lemmas} 

 \begin{lemma}[Corollary of Bellman's principle of optimality] \label{lem:bellman_optim}
Let $b^t_i$ be the belief over the state space derived from history $h^t_i$ and co-policy $\pi_{-i}$. Then,
\begin{equation}
    Q^*_{\pi_{-i}}(h^t_i, a_i) = \E_{s^t \samp b^t_i} [ R(s^t, a^t) + \E_{o_i \in P(\cdot, \cdot \mid s^t, a^t)} [\gamma \max_{a}(Q_{\pi_{-i}}^*(h^t_i + (a_i, o_i), a)))]]
\end{equation}
where 
 $a^t$ is the joint action specified by $a_i$ and the roots of $\pi^t_{-i}$.
\end{lemma}

\setcounter{theorem}{1}
\begin{theorem}[Bellman backup for strategic equivalence] \label{thm:bellman}
   For all histories $h^t_i$ and policies $\pi_{-i}, \pi'_{-i} \in \Pi^{h^t_i}_{-i}$, we have $\pi_{-i} \stackrel{\text{v}}{\sim}_{h^t_i} \pi'_{-i}$ if and only if,
  \begin{enumerate}[label=(\roman*)]
    \item \label{thm:bellman:cond1} $\underset{a_i}{\argmax} Q_i^*(\pi_{-i}, h^t_i, a_i) = \underset{a_i}{\argmax} Q_i^*(\pi'_{-i},h^t_i, a_i)$, and $\underset{a_i}{\max} Q_i^*(\pi_{-i}, h^t_i, a_i) = \underset{a_i}{\max} Q_i^*(\pi'_{-i},h^t_i, a_i)$;
    \item \label{thm:bellman:cond2} $\pi_{-i} \stackrel{\text{v}}{\sim}_{h^t_i + (a_i,o_i)}\pi'_{-i}$ for all actions $a_i \in \argmax_{a_i} Q_i^*(\pi_{-i},h^t_i, a_i)$ and corresponding possible observations $o_i$.
  \end{enumerate}
\end{theorem}

\begin{proof}
We will show that $\pi_{-i} \stackrel{\text{v}}{\sim}_{h^t_i} \pi'_{-i}$, i.e., $\br_{h^t_i}(\pi_{-i}) = \br_{h^t_i}(\pi'_{-i})$ if and only if \ref{thm:bellman:cond1} and \ref{thm:bellman:cond2} hold.
Condition \ref{thm:bellman:cond1} ensures that the best response actions to $\pi$ and $\pi'$ are the same at timepoint $t$, given history $h^t_i$. Condition \ref{thm:bellman:cond2} recursively checks that all other parts of the best response policies are the same:
\label{cond:main_thm_ii} $\pi_{-i} \stackrel{\text{v}}{\sim}_{h_i + (a_i,o_i)}\pi'_{-i}$ means that $\br_{h_i + (a_i,o_i)}(\pi_{-i}) = \br_{h_i + (a_i,o_i)}(\pi'_{-i})$ for all actions $a_i \in \argmax_{a_i} Q_i^*(\pi_{-i},h^t_i, a_i)$ and possible observations $o_i$, ensuring that no matter which observation agent $i$ observes, the set of best-response subpolicies are the same. \\
\end{proof}

\subsection{Computing strategic equivalence relations over subsets of the policy space}\label{appendix:other_subsets}
In some settings it is known that agents will follow some subset of the full policy space. In our examples, we assume that agents have some degree of rationality, so they play Nash-equilibrium or best-response policies. In these cases, it makes sense to restrict your strategic equivalence relations to the relevant subset of policies, potentially eliminating extraneous equivalence classes and thereby simplifying the strategy space. 

One way of computing the SER over a subset of the full policy space is by computing the SER over the full policy space and then discarding unwanted policies. This is usually undesirable, because it can introduce lots of unnecessary computation, especially in dynamic games. Instead, we simultaneously compute the relevant subset of policies along with the SER over them. In order to compute the SER over a subset of the policies in this way, a key characteristic is that the subset of policies have a recursive substructure that can be exploited during backwards recursion. Best-response policies and subgame perfect Nash equilibrium have the property that policies at time $t$ can be computed in terms of the policies at time $t+1$.

\subsection{Remarks on Algorithm \ref{alg:decpomdp_main}} \label{append:alg1_remarks}

\paragraph{Optimization in deterministic environments.}
Line 5 in algorithm \ref{alg:decpomdp_main} enumerates over all past joint policies to generate the reachable set of belief points. In deterministic environments, for any two joint pure policies, we can guarantee that only a single sequence of states with be traversed -- parts of the policies off this sequence have no effect. Therefore, we only need to iterate over joint policy chains (i.e., sequences of actions) to get all of the reachable belief states. The space of policy chains is much smaller than the space of full policy trees, providing a computational speedup in deterministic environments. All implementations of the algorithms (including all baselines) make use of this optimization.

\paragraph{Computational speedup in DecPOMDPs.}
Along with providing an efficient method for computing strategic equivalence classes, Algorithm \ref{alg:decpomdp_main} implements two strict improvements over traditional point-based, dynamic programming approaches for DecPOMDPs \cite{szer2006pointbased}.
\begin{enumerate}[label=(\roman*)]
  \item The $\argmax$ on line 6 of the algorithm proposed in \cite{szer2006pointbased} only needs to be computed over the current timestep's possible actions (rather than the full space of $t$-step policy trees) since we record best responses and their associated value functions for future histories.
  \item By computing strategic equivalence classes associated with future histories, we can reduce the total number of co-policies that need to be considered during the computation. Specifically, we only need to consider a single policy for each element of each of the strategic equivalence class, since all other policies will necessary induce the same best response.
\end{enumerate}

\paragraph{Computing subgame perfect Nash Equilibrium.}
In order to compute only the subgame perfect Nash equilibrium (rather than all of the best-response policies), we can do iterated elimination of strictly dominated strategies (IESDS) at each stage of the backwards induction. This will leave only the subgame perfect Nash equilibrium since IESDS eliminates all of the non-Nash equilibrium in common-payoff games.

 \subsection{Computing strategic equivalence relations in fully-observed settings} \label{sec:compute_info}

 Here, we provide additional context to computing the valued strategic equivalence relation (VSER) in fully-observed DecPOMDPs.

 In fully-observed DecPOMDPs, optimal policies and $Q$-functions can be Markovian. Therefore, we consider the set of intermediate VSERs $\stackrel{\text{v}}{\sim}_s$ over the subgames starting at states $s$. Let $T(s) = \{s' \mid \exists a \in \actions \text{ s.t. } P(s,a,s') > 0\}$ denote the successors of state $s$.
Then, $\pi_{-i} \stackrel{\text{v}}{\sim}_s \pi_{-i}'$ if and only if,
\begin{enumerate}[label=(\roman*)]\setlength\itemsep{0em}
    \item $\underset{a_i}{\argmax}\  Q^*_i(\pi_{-i},s,a_i)  =  \underset{a_i}{\argmax}\ Q^*_i(\pi_{-i}',s,a_i)$ 
    and $\underset{a_i}{\max}\ Q^*_i(\pi_{-i},s,a_i) = \underset{a_i}{\max}\  Q^*_i(\pi_{-i}',s,a_i)$; 
    \item $\pi_{-i} \stackrel{\text{v}}{\sim}_{s'} \pi_{-i}', \ \forall s' \in T(s)$.
\end{enumerate}
Since $Q^*_i(\pi_{-i}, s, a_i) = \E_{s'}[ R(s, a_i, \pi_{-i}(s)) + \gamma \max_{a'_i} Q^*_i(\pi_{-i}, s', a'_i)]$, the value function is also recursively defined in terms of successor states $T(s)$.

Given this recursive construction of the VSER of the subgames of $G$, we can compute the VSER over the subgame perfect Nash equilibrium of the full game $G$ using backwards induction similar to Algorithm \ref{alg:decpomdp_main}. The two key differences are (1) we only need to do backwards induction over the state space, not the history, and (2) Nash equilibrium can be computed directly, since the actions of co-policies are determined by the current state, rather than having to resort to iterated elimination.

\section{Additional Experiments} \label{append:extra_overcooked}

\begin{figure*}
     \centering
     \subfloat[\textbf{Overcooked -- ``Cramped'':} cooks need to coordinate on low-level movements to avoid getting in each other's way.]{\makebox[23em][c]{\includegraphics[width=.35\linewidth]{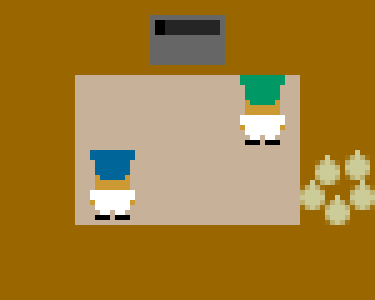}}\label{fig:overcooked_cramped}} 
     \hfill
     \subfloat[The time (lines) and number of argmax evaluations (bars) it take to compute the set of best-response policies using our algorithm (Algorithm \ref{alg:decpomdp_main}) and the baseline from \cite{szer2006pointbased}. Beyond a horizon of two, the baseline timed-out at 30 minutes.]{\begin{tikzpicture}

\definecolor{color0}{rgb}{0.12156862745098,0.466666666666667,0.705882352941177}
\definecolor{color1}{RGB}{214,39,40}

\begin{axis}[
mark size=3,
height=5.5cm,
width=7cm,
axis y line=right,
log basis y={10},
tick align=outside,
x grid style={white!69.0196078431373!black},
xmin=0.4875, xmax=6.2625,
xtick pos=left,
xtick style={color=black},
y grid style={white!69.0196078431373!black},
ylabel={Number of argmax evaluations},
ymin=9.31106555936993, ymax=30000000,
ymode=log,
ytick pos=right,
ytick style={color=black},
yticklabel style={anchor=west}
]
\draw[draw=none,fill=color0] (axis cs:0.75,0) rectangle (axis cs:1,0);
\draw[draw=none,fill=color0] (axis cs:1.75,0) rectangle (axis cs:2,18);
\draw[draw=none,fill=color0] (axis cs:2.75,0) rectangle (axis cs:3,180);
\draw[draw=none,fill=color0] (axis cs:3.75,0) rectangle (axis cs:4,1638);
\draw[draw=none,fill=color0] (axis cs:4.75,0) rectangle (axis cs:5,14760);
\draw[draw=none,fill=color0] (axis cs:5.75,0) rectangle (axis cs:6,133230);
\draw[draw=none,fill=color1] (axis cs:1,0) rectangle (axis cs:1.25,0);
\draw[draw=none,fill=color1] (axis cs:2,0) rectangle (axis cs:2.25,18);
\draw[draw=none,fill=color1] (axis cs:3,0) rectangle (axis cs:3.25,9566100);
\end{axis}

\begin{axis}[
mark size=3,
height=5.5cm,
width=7cm,
log basis y={10},
tick align=outside,
tick pos=left,
legend cell align={left},
legend style={fill opacity=0.8, draw opacity=1, text opacity=1, draw=white!80!black},
legend entries = {Our algorithm, Baseline},
x grid style={white!69.0196078431373!black},
xlabel={Horizon},
xmin=0.4875, xmax=6.2625,
xtick style={color=black},
y grid style={white!69.0196078431373!black},
ylabel={Time (seconds)},
ymin=4.12709008193735e-06, ymax=1214.62879500865,
ymode=log,
ytick style={color=black}
]
\addlegendimage{only marks, color0}
\addlegendimage{only marks, color1}
\addplot [very thick, color0, mark=|]
table {%
1 1.00135803222656e-05
2 0.000508308410644531
3 0.00631284713745117
4 0.112706184387207
5 0.977018117904663
6 11.5587148666382
};
\addplot [very thick, color1, mark=|]
table {%
1 7.70092010498047e-05
2 0.00056004524230957
3 500.608402967453
};
\end{axis}

\end{tikzpicture}\label{fig:decpomdp_benchmark_argmax}}
     \caption{Additional experiments. }
\end{figure*}
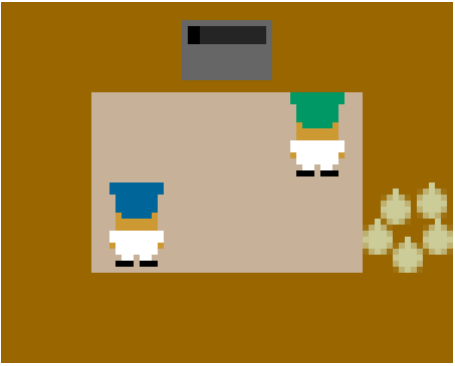
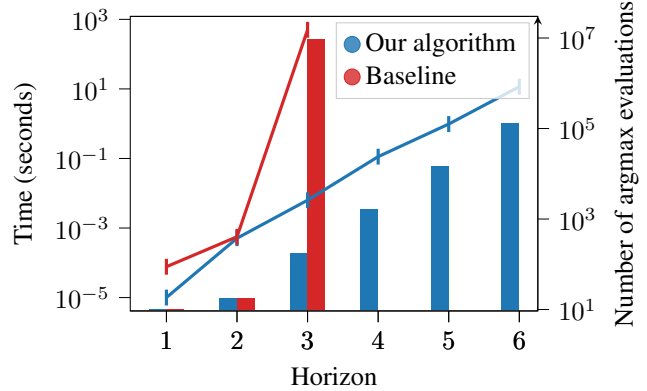

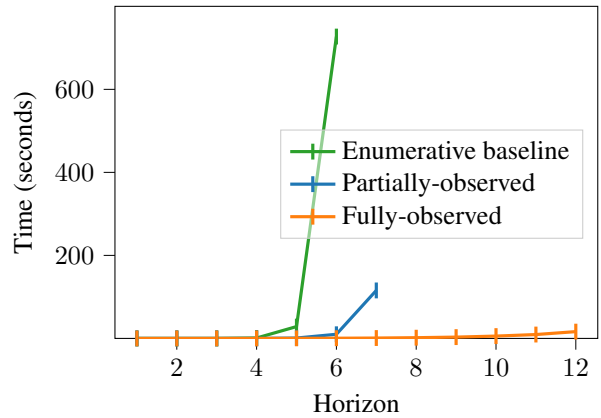
\begin{figure}
    \centering
    \begin{tikzpicture}

\definecolor{darkgray176}{RGB}{176,176,176}
\definecolor{darkorange25512714}{RGB}{255,127,14}
\definecolor{forestgreen4416044}{RGB}{44,160,44}
\definecolor{steelblue31119180}{RGB}{31,119,180}

\begin{axis}[
mark size=3,
height=6cm,
width=8cm,
legend cell align={left},
legend style={fill opacity=0.5, draw opacity=1, text opacity=1, at={(0.97,0.3)}, anchor=south east, draw=white!80!black},
tick align=outside,
tick pos=left,
x grid style={black},
xlabel={Horizon},
xmin=0.45, xmax=12.55,
xtick style={color=black},
y grid style={black},
ylabel={Time (seconds)},
ymin=1.3435938144627e-06, ymax=800,
ytick style={color=black},
ytick={200,400,600},
yticklabels={
  \(\displaystyle {200}\),
  \(\displaystyle {400}\),
  \(\displaystyle {600}\),
}
]
\addplot [very thick, forestgreen4416044, mark=|]
table {%
1 0.000292629999999905
2 0.00148614299999994
3 0.0390521719999999
4 0.96808691
5 28.469508398
6 727.163952321
};
\addlegendentry{Enumerative baseline}
\addplot [very thick, steelblue31119180, mark=|]
table {%
1 3.50060000000596e-06
2 0.000451049200000009
3 0.0052005956
4 0.0725061821
5 0.8929676979
6 9.9627010962
7 115.5845832549
};
\addlegendentry{Partially-observed}
\addplot [very thick, darkorange25512714, mark=|]
table {%
1 3.95200000002927e-06
2 0.00140432599999996
3 0.010815425
4 0.056928686
5 0.189656735
6 0.469652239
7 0.895915125
8 1.736391487
9 3.123782996
10 5.584874926
11 9.312807397
12 16.335166951
};
\addlegendentry{Fully-observed}

\end{axis}

\end{tikzpicture}
    \caption{The time (on an absolute scale) for computing strategic equivalence relations over policies of varying horizon in the Overcooked environment depicted in Figure \ref{fig:overcooked_tiny}. \textit{Fully-observed} (Algorithm in Appendix \ref{sec:compute_info}) uses the fully-observed version of the environment. \textit{Partially-observed} (Algorithm \ref{alg:decpomdp_main}) and \textit{Enumerative baseline} (described in Section \ref{sec:benchmark_decpomdp_planning}) use the partially-observed version of the environment.}
    \label{fig:benchmark_nonlog}
\end{figure}

In this section, we provide two additional experimental results. In the first, we analyze the valued strategic equivalence relations (VSERs) over subgame perfect Nash equilibrium of an additional fully-observed Overcooked environment (depicted in \ref{fig:overcooked_cramped}). Videos of each of the SECs in this environment is provided in the supplemental material. As in the environment depicted in Figure \ref{fig:overcooked_schelling}, agents get +1 reward anytime an onion is placed into any pot and have access to six actions: move left, right, up, down, stay still, and interact. Optimal policies achieve a total payoff of +2.
Agents need to coordinate on low-level movements to avoid getting in each other's way to get as many onions in the pots within a horizon of 9. This environment gives rise to only two equivalence classes, depending on which agent collects the onions first. See the two videos \href{https://minknowledge.github.io/cramped}{on our website} for visualizations.

We also report an expanded version of Figure \ref{fig:decpomdp_benchmark} in the form of Figure \ref{fig:decpomdp_benchmark_argmax} that also reports the number of \textit{argmax evaluations} performed by our algorithm and the baseline from \cite{szer2006pointbased}. Within the dynamic programming step of either algorithm, an argmax is performed to find the optimal action (line 10 of Algorithm \ref{alg:decpomdp_main}) or subpolicy (step 2.a.iii of Figure 2 in \cite{szer2006pointbased}) for the current subgame. This measure provides a proxy for the number of iterations the inner-most loop of each algorithm needs to take and is recorded in the form of bars in Figure \ref{fig:decpomdp_benchmark_argmax}, following a trend similar to wall-clock time.

\end{document}